\documentclass{article}

\usepackage{microtype}
\usepackage{graphicx}
\usepackage{subcaption}
\usepackage{booktabs} %

\usepackage{hyperref}

 \usepackage[preprint]{icml2026}

\usepackage{amsmath}
\usepackage{amssymb}
\usepackage{mathtools}
\usepackage{amsthm}

\usepackage{booktabs}
\usepackage{multirow}
\usepackage{graphicx}
\usepackage[table]{xcolor}

\usepackage[capitalize,noabbrev]{cleveref}
\usepackage{xcolor}
\usepackage{colortbl}
\usepackage{pifont}   %
\usepackage{threeparttable}

\usepackage{tcolorbox}

\definecolor{sotaBlue}{HTML}{4A90E2} 

\newcommand{\cc}[1]{\cellcolor{blue!#1}}

\definecolor{deepblue}{RGB}{25, 50, 100}   %
\definecolor{softbg}{RGB}{248, 250, 252}   %

\tcbuselibrary{theorems, skins, breakable} %

\tcbset{
  commonstyle/.style={
    enhanced,
    boxrule=0.4mm,
    arc=0mm, %
    fonttitle=\bfseries\fontsize{9pt}{12pt}\selectfont,
    attach boxed title to top left={yshift=-2mm, xshift=1mm},
    boxed title style={rounded corners, arc=3pt},
    separator sign={~},
  }
}

\definecolor{NatObs}{HTML}{576FA0}  %
\definecolor{NatIns}{HTML}{A577AD}  %
\definecolor{NatFin}{HTML}{B57979}  %

\newtcbtheorem[]{observation}{Observation}{
  commonstyle,
  colframe=NatObs!75!black,
  colback=NatObs!10!white,
  coltitle=white,
  boxed title style={
    colback=NatObs,
    colframe=NatObs!75!black,
    coltext=white
  },
}{obs}

\newtcbtheorem[]{insight}{Insight}{
  commonstyle,
  colframe=NatIns!75!black,
  colback=NatIns!10!white,
  coltitle=white,
  boxed title style={
    colback=NatIns,
    colframe=NatIns!75!black,
    coltext=white
  },
}{ins}

\newtcbtheorem[]{finding}{Finding}{
  commonstyle,
  colframe=NatFin!75!black,
  colback=NatFin!10!white,
  coltitle=white,
  boxed title style={
    colback=NatFin,
    colframe=NatFin!75!black,
    coltext=white
  },
}{cor}

\newcommand{\xmark}{\ding{55}}        %
\newcommand{\halfmark}{\ensuremath{-}}

\definecolor{yesgreen}{RGB}{220,245,220}
\definecolor{nored}{RGB}{245,220,220}
\definecolor{halfyellow}{RGB}{255,245,200}
\definecolor{nogray}{RGB}{200,200,200}

\newcommand{\yescell}{\cellcolor{yesgreen}\checkmark}
\newcommand{\nocell}{\cellcolor{nored}\xmark}
\newcommand{\halfcell}{\cellcolor{nogray}\halfmark}

\theoremstyle{plain}
\newtheorem{theorem}{Theorem}[section]
\newtheorem{proposition}[theorem]{Proposition}

\theoremstyle{definition}
\newtheorem{definition}[theorem]{Definition}
\newtheorem{assumption}[theorem]{Assumption}
\theoremstyle{remark}

\usepackage[textsize=tiny]{todonotes}

\icmltitlerunning{A Tilted Seesaw: Revisiting Autoencoder Trade-off for Controllable Diffusion}

\begin{document}

\twocolumn[
  \icmltitle{A Tilted Seesaw: Revisiting Autoencoder Trade-off for Controllable Diffusion}

  \icmlsetsymbol{equal}{*}

  \begin{icmlauthorlist}
    \icmlauthor{Pu Cao}{bupt}
    \icmlauthor{Yiyang Ma}{bupt}
    \icmlauthor{Feng Zhou}{bupt}
    \icmlauthor{Xuedan Yin}{thu}
    \icmlauthor{Qing Song}{bupt}
    \icmlauthor{Lu Yang}{bupt}
  \end{icmlauthorlist}

  \icmlaffiliation{bupt}{Beijing University of Posts and Telecommunications, Beijing, China}
  \icmlaffiliation{thu}{Tsinghua University, Beijing, China}

  \icmlcorrespondingauthor{Lu Yang}{soeaver@bupt.edu.cn}

  \icmlkeywords{Autoencoder, Diffusion Model}

  \vskip 0.3in
]

\printAffiliationsAndNotice{}  %

\begin{abstract}
In latent diffusion models, the autoencoder (AE) is typically expected to balance two capabilities: faithful reconstruction and a generation-friendly latent space (e.g., low gFID).
In recent ImageNet-scale AE studies, we observe a systematic bias toward generative metrics in handling this trade-off: reconstruction metrics are increasingly under-reported, and ablation-based AE selection often favors the best-gFID configuration even when reconstruction fidelity degrades.
We theoretically analyze why this gFID-dominant preference can appear unproblematic for ImageNet generation, yet becomes risky when scaling to controllable diffusion: AEs can induce condition drift, which limits achievable condition alignment.
Meanwhile, we find that reconstruction fidelity, especially instance-level measures, better indicates controllability.
We empirically validate the impact of tilted autoencoder evaluation on controllability by studying several recent ImageNet AEs. Using a multi-dimensional condition-drift evaluation protocol reflecting controllable generation tasks, we find that gFID is only weakly predictive of condition preservation, whereas reconstruction-oriented metrics are substantially more aligned. ControlNet experiments further confirm that controllability tracks condition preservation rather than gFID.
Overall, our results expose a gap between ImageNet-centric AE evaluation and the requirements of scalable controllable diffusion, offering practical guidance for more reliable benchmarking and model selection.
\end{abstract}

\section{Introduction}

Latent diffusion models~\cite{rombach2022high} have become a dominant paradigm for high-quality image generation, largely because they perform denoising in a compact latent space rather than directly in pixel space.
In this pipeline, an autoencoder (AE) \cite{kingma2013auto} defines the latent representation, and a diffusion model is trained to model the data distribution by denoising samples from a prescribed noise process in that latent space.
Consequently, advances in autoencoder design and training have recently become a key driver of progress in latent diffusion.

Autoencoders face a generation–reconstruction trade-off: making the latent space more generation-friendly can improve diffusion sampling and visual realism, while high-fidelity reconstruction requires preserving instance-specific structure \cite{yao2025reconstruction,fan2025prism}. ImageNet-scale benchmarks \cite{deng2009imagenet} provide a relatively low-cost way to test how an autoencoder interface affects latent diffusion, but they also tend to steer optimization and comparison toward generative scores (e.g., gFID \cite{heusel2017gans}). In practice, we observe a tilt in both metric reporting and model selection that down-weights reconstruction: many papers report only distribution-level reconstruction metrics (e.g., rFID) or omit reconstruction evaluation altogether, and ablation decisions often prioritize the best generative score even when instance-level fidelity degrades. The issue is that these practices can make instance-level structural preservation largely invisible. As a result, while this tilted “seesaw” improves ImageNet generative scores, it raises risks when moving to open-world controllable generation, where control relies on preserving input-specific structure and autoencoder-induced drift can directly undermine condition adherence and generalization.

We theoretically analyze how this tilted evaluation practice impacts the scaling of latent diffusion toward large controllable generation models.
We characterize the discrepancy between the condition extracted from an input image and that extracted from its reconstruction as \emph{condition drift}.
Our analysis shows that autoencoder-induced condition drift becomes a bottleneck for controllable diffusion. It introduces an irreducible control error, so that even if the diffusion backbone fits the latent conditional distribution perfectly, the achievable condition alignment is still lower-bounded by the drift introduced by the autoencoder interface.
This implies that an autoencoder with large condition drift can make a diffusion model hard to use for controllable generation, even when its unconditional generative quality is strong.
Finally, we connect condition drift to reconstruction evaluation.
We argue that reconstruction metrics are critical for diagnosing drift, and that instance-level fidelity metrics are particularly important. Compared to distribution-level reconstruction scores such as rFID, they more directly reflect condition drift and better safeguard against drift-driven mis-selection.

To validate our observations and analysis, we study several recent ImageNet autoencoders and find that condition preservation reveals a gap in ImageNet-centric evaluation.
We build a condition-drift evaluation protocol spanning conditions from a range of controllable generation tasks to better reflect real-world settings.
Using rank-based correlation analysis, we find that gFID is only weakly correlated with both instance-level reconstruction metrics and many condition-consistency measures.
In contrast, reconstruction-oriented signals (especially rFID and simple instance-level metrics such as PSNR) align better with condition preservation.
We further corroborate this conclusion by training ControlNet \cite{zhang2023adding} with two representative autoencoders, where controllable generation quality tracks condition preservation rather than gFID.
Finally, we probe latent-space condition predictability to test whether control-relevant cues are already discarded at encoding time.

Taken together, our study revisits autoencoder evaluation for latent diffusion from three perspectives:
\begin{itemize}
	\item
	\textbf{Observations.}
	We show that reconstruction is often under-reported and that ablation-based selection can favor the best gFID despite degraded reconstruction.

    \item
    \textbf{Analysis.}
    We demonstrate that gFID-dominant evaluation can mask AE-induced condition drift in controllable diffusion, and that reconstruction fidelity provides a practical proxy for diagnosing such drift.

    \item
    \textbf{Empirical study.}
    We quantify how popular metrics relate to condition preservation across ImageNet-scale autoencoders, and identify when gFID-driven selection becomes misaligned with controllability.
\end{itemize}

Based on these findings, we recommend a compact reporting and selection protocol: report generative quality together with at least one instance-level reconstruction metric and explicit condition consistency for the widely used target controls.

\section{Related Work}
\subsection{Latent Diffusion Models}
Latent diffusion models (LDMs)~\cite{rombach2022high} reduce the computational cost of pixel-space diffusion by introducing autoencoders to construct a compact latent space. Specifically, images are encoded into latents where the denoising process is performed, and then decoded back to the pixel space to enable high-resolution synthesis. This paradigm has been widely adopted by modern large-scale text-to-image diffusion models, such as SDXL\cite{podellsdxl}, Stable Diffusion 3~\cite{esser2024scaling}, and FLUX~\cite{flux2024}. Owing to their strong generative capacity, LDMs are also extensively used for downstream controllable generation tasks\cite{11304732}, including spatial control\cite{zhang2023adding}, domain adaptation\cite{cao2025image}, and image editing\cite{huang2025e4c}.

\subsection{Improved Autoencoders in Latent Diffusion}
Autoencoder design for latent diffusion has progressed from vector-quantized tokenizers such as VQGAN to higher-capacity hierarchical VAEs like NVAE, with the shared goal of improving reconstruction fidelity and perceptual quality in latent diffusion pipelines \cite{esser2021taming,vahdat2020nvae,svgt2i2025}. Recent studies suggest that injecting semantic structure into the latent space can effectively enhance generative performance\cite{vtp,shi2025latentdiffusionmodelvariational,leng2025repae}. For instance, VA-VAE encourages the VAE to align with a vision foundation model\cite{yao2025reconstruction}, while RAE directly employs a frozen DINOv2 encoder\cite{zheng2025diffusion}. These approaches report notable improvements in generative quality, as reflected by improved gFID.

\section{Preliminaries}
\label{sec:prelim}

\subsection{Controllable Latent Diffusion Model}
\label{sec:ae}

Many modern diffusion models operate in the latent space of an autoencoder (AE) for efficiency~\cite{rombach2022high}.
Given an image $x\in\mathcal{X}$, the encoder maps it to a latent code
\begin{equation}
    z_0=\mathcal{E}(x)\in\mathcal{Z},
\end{equation}
and reconstruction is obtained by decoding $\hat{x}=\mathcal{D}(z_0)$.
We consider controllable generation where a condition $c$ specifies desired attributes of the corresponding image (e.g., textual description, edges, depth, identity).
Training uses paired data $(x,c)\sim p_{\text{data}}(x,c)$.
Let $z_0=\mathcal{E}(x)$ and let $z_t$ denote a perturbed/noised version of $z_0$ at step $t$~\cite{ho2020denoising,liu2022flow}.
The conditional denoiser/inverter $f_\theta$ (e.g., UNet or DiT) is trained by
\begin{equation}
    \mathcal{L}_{\mathrm{gen}}
    = \mathbb{E}_{(z,c),\,t,u}
    \bigl[\|u - f_\theta(z_t,t,c)\|\bigr],
\end{equation}
where $u$ is the training target (e.g., noise/velocity/clean latent, depending on the formulation).
This objective learns a conditional latent generator (informally, $p_\theta(z_0\mid c)$), and generated samples are mapped back to images through the AE decoder $\mathcal{D}$.

\subsection{Evaluation Metrics and Trade-offs}
\label{sec:prelim_metrics}

In this part, we briefly review the standard evaluation protocols and \textbf{trade-off between generation and reconstruction} for autoencoders in recent literature.

\textbf{gFID (generation FID)} calculates the Fréchet Inception Distance between real data and samples generated by a diffusion model trained on the specific autoencoder. Since gFID directly measures the autoencoder's effectiveness in supporting downstream generative tasks, it is prioritized by many recent works as the primary indicator.

\textbf{Reconstruction metrics} fall into two categories. \textbf{rFID (reconstruction FID)} evaluates the distributional consistency between the entire input dataset and the reconstructed dataset. Besides, {instance-level metrics}, such as PSNR, SSIM\cite{wang2004image}, and LPIPS\cite{zhang2018unreasonable}, measure the fidelity of independent samples by quantifying the pixel-wise or perceptual discrepancy between a specific input $x$ and its reconstruction $\hat{x}$.

Notably, recent studies~\cite{yao2025reconstruction,zheng2025diffusion} have identified a significant trade-off between generation and reconstruction: autoencoders optimized for superior reconstruction often exhibit degraded generation capability (e.g., worse gFID). 
However, in the next section, we show that recent work exhibits an evaluation bias in this trade-off, in which the ability to generate is systematically overemphasized.

\section{Dominance of gFID in AE Evaluation}
\label{sec:imagenet_observations}

To reveal a growing evaluation bias in recently improved autoencoders, we examine several representative improved autoencoders that report strong performance, including VA-VAE\cite{yao2025reconstruction}, REPA-E\cite{leng2025repae}, RAE\cite{zheng2025diffusion}, VTP\cite{vtp}, SVG\cite{shi2025latent}, UAE\cite{fan2025prism}, FAE\cite{gao2025one}, and SVG-T2I\cite{svgt2i2025}.
While reconstruction used to be the central capability of an autoencoder, we observe a clear shift in emphasis: many recent designs prioritize generative quality (i.e., gFID) and, often implicitly, treat reconstruction as secondary.
This section summarizes two empirical observations from how these works report metrics and how they choose final configurations during ablations.
Unless otherwise noted, all models in this table are studied and evaluated on ImageNet, with SVG-T2I being the only exception.

\begin{table}[t!]
\centering
\caption{\textbf{Overview of evaluation metrics and ablation configurations in recent autoencoders.} The left group indicates metrics reported for the final model. The right group marks whether the selected ablation setting corresponded to the best generative (gFID) or reconstruction (rFID) quality. Cells with gray backgrounds indicate the metric was not reported or the model was not included in the ablation analysis.}
\label{tab:merged_autoencoder_analysis}
\small
\begin{threeparttable}
\renewcommand{\arraystretch}{1.2} %

\resizebox{\columnwidth}{!}{%
\begin{tabular}{l c @{\hspace{6pt}\vrule\hspace{6pt}} c c c c @{\hspace{6pt}\vrule\hspace{6pt}} c c}
    \toprule
    \multirow{2}{*}{\textbf{Model}} & \multirow{2}{*}{\textbf{Date}} & 
    \multicolumn{4}{c}{\textbf{Reported Metrics}} & 
    \multicolumn{2}{c}{\textbf{Ablation Selection}} \\
    
    \cmidrule(lr){3-6} \cmidrule(l){7-8}
    
     & & \textbf{rFID} & \textbf{PSNR} & \textbf{LPIPS} & \textbf{SSIM} & \textbf{gFID} & \textbf{rFID} \\
    \midrule
    
    VA-VAE  & Jan 25 & \yescell & \yescell & \yescell & \yescell & \yescell & \nocell \\
    REPA-E  & Apr 25 & \yescell & \nocell  & \nocell  & \nocell  & \yescell & \nocell \\
    RAE     & Oct 25 & \yescell & \nocell  & \nocell  & \nocell  & \yescell & \nocell \\
    VTP     & Oct 25 & \yescell & \nocell  & \nocell  & \nocell  & \yescell & \yescell \\
    SVG     & Oct 25 & \yescell & \nocell  & \nocell  & \nocell  & \yescell & \nocell \\
    UAE     & Dec 25 & \yescell & \yescell & \yescell & \nocell  & \halfcell  & \halfcell \\ %
    FAE     & Dec 25 & \nocell  & \nocell  & \nocell  & \nocell  & \halfcell  & \halfcell \\ %
    SVG-T2I & Dec 25 & \nocell  & \nocell  & \nocell  & \nocell  & \halfcell  & \halfcell \\ %
    
    \bottomrule
\end{tabular}%
}
\end{threeparttable}
\end{table}

\paragraph{Reconstruction quality is increasingly under-reported.}
The left part in table~\ref{tab:merged_autoencoder_analysis} summarizes whether each paper reports reconstruction-related metrics, restricted to those used in the main text to compare methods and baselines.
Here, rFID measures distribution-level consistency between the original images and reconstructions, whereas PSNR/LPIPS/SSIM are instance-level metrics that reflect per-sample fidelity.
Across eight works, two of them do not evaluate reconstruction at all. Meanwhile, four papers report only rFID, while two papers report both distribution-level and instance-level reconstruction metrics.

\begin{observation}{Under-evaluated Reconstruction}{}
Recent high-performing autoencoders often under-reported reconstruction metrics, especially instance-level reconstruction metrics.
\end{observation}

\paragraph{Model selection systematically favors generation over reconstruction.}
Beyond metric reporting, we further observe a systematic bias in model selection.
In the right panel of Table~\ref{tab:merged_autoencoder_analysis}, we examine five papers that perform ablations over key autoencoder design choices (e.g., encoder families or architectural backbones).
Across all five, the chosen configuration is the one with the best gFID, i.e., the strongest generative performance.
Only VTP selects a final model that is simultaneously optimal for both generation and reconstruction metrics within its ablation pool.
This tendency is particularly pronounced in works that replace or augment the encoder with semantic backbones.
For example, several works\cite{zheng2025diffusion,yao2025reconstruction} ablate the choice of semantic encoder and demonstrate that MAE-based variants substantially improve reconstruction metrics\cite{he2022masked}, whereas DINO-based variants score better on generation and semantic probing\cite{oquab2023dinov2}.
Nevertheless, the final selection is frequently driven by the generative score alone.

\begin{observation}{Biased ablation decisions}{}
Recent works often favor the variant with the best generative score (gFID) despite clear reconstruction gains from alternatives.
\end{observation}

\paragraph{Takeaway.}
These observations suggest that current AE benchmarking implicitly treats reconstruction as secondary once generative performance is strong.
In the next section, we show theoretically that this tilted evaluation can induce an irreversible degradation in controllable generation.

\section{Autoencoder-Induced Condition Drift}
\label{sec:cc_bottleneck}

The previous section highlighted a tilted evaluation practice that often de-emphasizes reconstruction in favor of generative scores.
In this section, we theoretically analyze how the autoencoder induces condition drift, which degrades controllable generation in diffusion models.
We show that this drift induces an alignment limit at the latent optimum, and then explain how reconstruction fidelity constrains and serves as a practical proxy for evaluating drift, which is a safeguard that tilted evaluation can undermine.

\subsection{Formulation}
\label{sec:cc_formulation}

Let $x \sim p_{\text{data}}(x)$ be an image and let $c \in \mathcal{C}$ be a control signal derived from $x$.
For image-grounded controls, the condition is obtained by a deterministic projector $\phi:\mathcal{X}\to\mathcal{C}$ (e.g., canny edges, depth, segmentation), so that $c=\phi(x)$.
An autoencoder consists of an encoder $\mathcal{E}:\mathcal{X}\to\mathcal{Z}$ and decoder $\mathcal{D}:\mathcal{Z}\to\mathcal{X}$, producing the reconstruction
$\hat{x}=(\mathcal{D}\circ \mathcal{E})(x)$.

\begin{definition}[Condition Drift]
\label{def:condition_drift}
Given a projector $\phi$, the autoencoder-induced condition drift is
\begin{equation}
    \Delta_{\mathrm{AE}}(x) \;=\; \|\phi(x) - \phi(\hat{x})\|.
    \label{eq:cc_delta}
\end{equation}
\end{definition}

$\Delta_{\mathrm{AE}}(x)$ measures instance-level structural preservation, i.e., preserving the instance-wise correspondence between $x$ and $\hat{x}$ that is relevant to downstream conditions: it is zero if and only if the condition extracted from the reconstruction matches that of the original.
Throughout, we interpret $\Delta_{\mathrm{AE}}(x)$ as condition drift induced by the autoencoder under the control representation defined by $\phi$.

\subsection{Irreducible Condition Drift from Autoencoders}
\label{sec:cc_irreducible}

\begin{figure}[t!]
    \centering
    \includegraphics[width=\columnwidth]{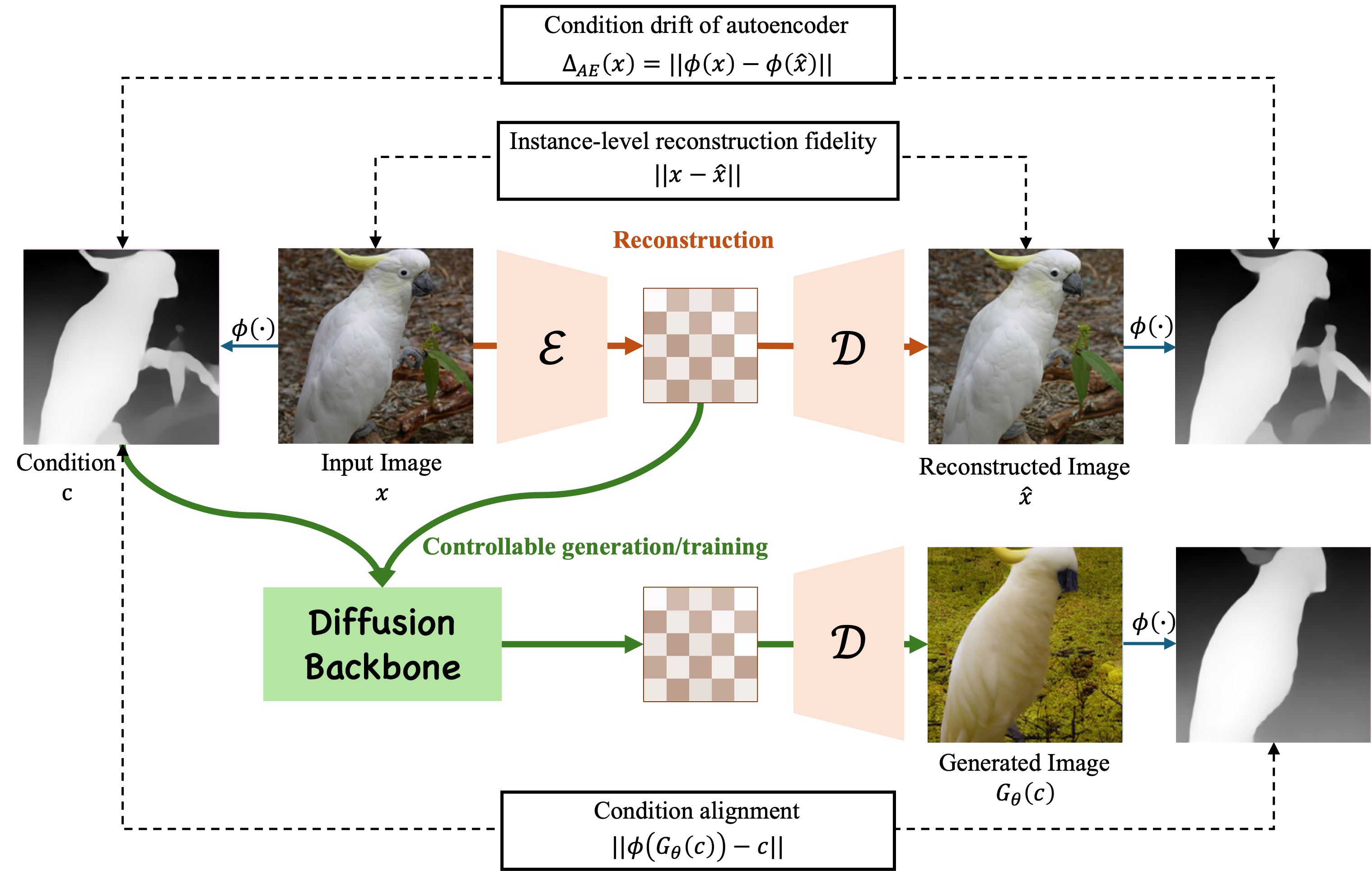}
    \caption{\textbf{Illustration of how autoencoder affects controllable diffusion generation.} We utilize RAE as example, where the controllable generation is realized by trained ControlNet.}
    \label{fig:condition_drift}
\end{figure}

In controllable latent diffusion, given a condition $c$, the diffusion backbone samples a latent code $z \sim p_\theta(z\mid c)$, which is then mapped back to image space by the decoder $x=\mathcal{D}(z)$.
For convenience, we denote the resulting conditional generator by
\[
G_\theta(c) \;=\; \mathcal{D}(z), \qquad z \sim p_\theta(z\mid c),
\]
so that $G_\theta(c)$ is a random image sample conditioned on $c$ and we suppose $G_\theta(c)\sim p_\theta(x\mid c)$.

To quantify downstream controllability, we measure how well generated samples satisfy the desired condition via the conditional alignment error
\begin{equation}
    \mathcal{L}_{\mathrm{align}}(\theta)
    \;=\;
    \mathbb{E}_{c\sim p(c)}\Bigl[\bigl\|\phi(G_\theta(c)) - c\bigr\|\Bigr],
    \label{eq:cc_align_def}
\end{equation}
where $p(c)$ is induced by $c=\phi(x)$ with $x\sim p_{\text{data}}(x)$. We interpret $\mathcal{L}_{\mathrm{align}}$ as a controllability metric for trained diffusion models (Figure~\ref{fig:condition_drift}, bottom). In this subsection, we show that controllability is fundamentally limited by autoencoder-induced condition drift.

\paragraph{Objective shift induced by the autoencoder.}
In controllable generation, the goal is straightforward: given a condition $c$, we want the generator to produce images that satisfy $c$, i.e., $G_\theta(c)$ should follow the desired conditional distribution $p(x\mid c)$.
With a fixed autoencoder interface, however, the condition that is actually obtained after decoding can differ from the nominal condition.
Concretely, when we encode an image with condition $c$ and then decode it back, the reconstruction may correspond to a slightly different condition, which we denote by $c'$.
This is exactly the autoencoder-induced condition drift.

As a result, the supervision used in latent diffusion is subtly mis-specified.
Although training is organized by the nominal label $c$, the samples that can be produced through the decoder are effectively drawn from the reconstruction-induced conditional distribution for $c'$.
In shorthand, autoencoder drift tilts the learning target from
\[
G_\theta(c)\sim p(x\mid c)
\qquad \text{toward} \qquad
G_\theta(c)\sim p(x\mid c').
\]
This shift is harmless only when $c'=c$. Otherwise it creates an irreducible mismatch between the desired condition and what the fixed autoencoder can faithfully realize.

This mismatch becomes a hard limit in the latent-optimal regime.
When the diffusion backbone perfectly fits the latent conditional distribution induced by the fixed autoencoder, generation reproduces the decoded distribution associated with the training pairs for label $c$.
In that idealized limit, the remaining alignment error is entirely due to the autoencoder interface, and is governed by the discrepancy between $c$ and the produced $c'$.

\begin{theorem}[Alignment Limit at Latent Optimum]
\label{thm:cc_floor}
Fix an autoencoder $(\mathcal{E},\mathcal{D})$ and a condition projector $\phi$.
Assume the diffusion backbone perfectly fits the induced latent conditional distribution, i.e., $p_\theta(z\mid c) = p_{\mathcal{E}}(z\mid c)$, and generation follows $G_\theta(c)=\mathcal{D}(z)$ with $z\sim p_\theta(z\mid c)$.
Then the expected alignment error equals the expected autoencoder-induced condition drift:
\begin{equation}
    \mathcal{L}_{\mathrm{align}}(\theta)
    \;=\;
    \mathbb{E}_{x\sim p_{\text{data}}}\bigl[\Delta_{\mathrm{AE}}(x)\bigr].
    \label{eq:cc_floor_eq}
\end{equation}
\end{theorem}
A proof is provided in Appendix~\ref{app:proofs}. Theorem~\ref{thm:cc_floor} indicates that, no matter how much we scale the diffusion backbone to better model the data distribution in latent space, controllability remains fundamentally limited by the condition drift induced by the fixed autoencoder.

\begin{insight}{Condition Drift Limits Controllability}{cc_bottleneck}
Autoencoder-induced condition drift constitutes a bottleneck for controllability in latent diffusion.
\end{insight}

\subsection{How to Evaluate Condition Drift}
\label{sec:cc_vs_metrics}

In practice, the projector $\phi$ is typically unknown, since the large diffusion model is expected to serve as a general-purpose generative prior that can be applied to arbitrary controllable generation tasks.
We therefore seek condition-agnostic signals that indicate whether the encode--decode mapping preserves instance-specific structure, i.e., whether it is likely to keep $\phi(x)$ and $\phi(\hat{x})$ close for a broad family of projectors.
As we show below, instance-level reconstruction fidelity provides a conservative constraint on drift under mild stability assumptions, whereas rFID is a marginal distribution metric that does not constrain instance-wise coupling. Thus, tilted evaluation that de-emphasizes reconstruction can weaken safeguards against drift.

\paragraph{Instance-level fidelity as a conservative guardrail.}
A common abstraction is to assume that $\phi$ is locally stable on the natural image manifold under a chosen image metric, captured by a Lipschitz condition.

\begin{assumption}[Stability of the projector]
\label{ass:cc_lipschitz}
Assume the projector $\phi:\mathcal{X}\to\mathcal{C}$ is $K_\phi$-Lipschitz on the image manifold (under the chosen norm on $\mathcal{X}$):
$\|\phi(x)-\phi(\tilde{x})\| \le K_\phi\|x-\tilde{x}\|$.
\end{assumption}

Under this assumption, per-instance reconstruction error upper-bounds drift:
\begin{equation}
    \Delta_{\mathrm{AE}}(x)=\|\phi(x)-\phi(\hat{x})\|
    \le K_\phi\cdot \|x-\hat{x}\|.
    \label{eq:cc_inst_bound}
\end{equation}

This bound is intentionally simple: it does not aim to tightly predict drift for a specific $\phi$.
Instead, it motivates instance-level reconstruction fidelity as a conservative safeguard in open-world settings:
reducing per-image reconstruction error reduces an upper bound on drift for stable projectors.
In particular, pixel-space metrics such as MSE/PSNR directly correspond to $\|x-\hat{x}\|$.
Other instance-level metrics (e.g., SSIM or LPIPS) do not follow from Eq.~\eqref{eq:cc_inst_bound}, but can still be informative in practice. We therefore treat them as empirical proxies and evaluate their relationship with drift in Section~\ref{sec:gap_imagenet_eval}.

\begin{insight}{Instance-level fidelity as a guardrail}{instance_consistency}
Instance-level reconstruction fidelity provides a conservative safeguard against condition drift under stability assumptions.
\end{insight}

\paragraph{rFID measures marginal reconstruction quality rather than instance-wise coupling.}
Let $T=\mathcal{D}\circ\mathcal{E}$ denote the encode--decode mapping and let $p_{\hat{x}} = T_{\#}p_x$ be the induced marginal distribution of reconstructions.
By definition, rFID compares the feature-space marginals of real images and reconstructions, and can be viewed as a distance between $p_x$ and $p_{\hat{x}}$.
Crucially, rFID depends only on the marginal distribution of reconstructions and is blind to how each input $x$ is paired with its reconstruction $\hat{x}=T(x)$.
In contrast, condition drift $\Delta_{\mathrm{AE}}(x)=\|\phi(x)-\phi(\hat{x})\|$ is inherently coupling-dependent: it is defined on paired samples and cannot be determined from marginals alone.

\begin{proposition}[Marginal Matching Does Not Identify Drift]
\label{prop:cc_marginal}
Let $x \sim p_x$ be real images and let $\hat{x}$ be reconstructions with marginal distribution $p_{\hat{x}}$.
If $p_{\hat{x}} = p_x$, then the population rFID is $0$.
However, the expected condition drift $\mathbb{E}\|\phi(x)-\phi(\hat{x})\|$ depends on the coupling (i.e., the joint relationship) between $x$ and $\hat{x}$, and can vary widely even when rFID is perfect.
\end{proposition}

Proposition~\ref{prop:cc_marginal} highlights a fundamental non-identifiability: marginal reconstruction quality does not determine instance-wise structure preservation.
An autoencoder may achieve low rFID as long as $T$ acts as a distribution mapper that pushes the reconstruction marginal $p_{\hat{x}}$ toward the data marginal $p_x$, even if it fails to preserve the structure of individual samples and thus induces nontrivial drift. It indicates that rFID may be risky to constraint condition drift.

\begin{insight}{rFID is risky for condition drift}{rfid_consistency}
rFID constrains reconstruction quality at the level of marginal feature distributions and can be a useful heuristic for reconstruction, but low rFID alone cannot certify small condition drift because it does not constrain instance-wise coupling.
\end{insight}

\begin{table*}[t]
\centering
\caption{\textbf{Empirical study on recent powerful autoencoders}. Background color intensity indicates performance.$^\dagger$: results borrowed from the official reports.}
\label{tab:main_result}
\renewcommand{\arraystretch}{1.2} % 增加行高，更透气
\resizebox{\textwidth}{!}{%
\begin{tabular}{lccccccccccccc|cc}
\toprule
\multirow{2}{*}{\textbf{Model}} & \multicolumn{4}{c}{\textbf{Generation Metrics}$^\dagger$} & \multicolumn{4}{c}{\textbf{Reconstruction Metrics}} & \multicolumn{5}{c}{\textbf{Condition Consistency}} & \multicolumn{2}{c}{\textbf{Latent Condition}} \\
\cmidrule(lr){2-5} \cmidrule(lr){6-9} \cmidrule(lr){10-14} \cmidrule(lr){15-16}
 & gFID $\downarrow$ & IS $\uparrow$ & Prec $\uparrow$ & Rec $\uparrow$ & rFID $\downarrow$ & PSNR $\uparrow$ & SSIM $\uparrow$ & LPIPS $\downarrow$ & Spatial $\downarrow$ & Identity $\uparrow$ & Face@R $\uparrow$ & CLIP $\uparrow$ & DINO $\uparrow$ & Canny $\downarrow$ & Depth $\downarrow$ \\
\midrule
VA-VAE & \cc{15}2.17 & \cc{5}205.6 & \cc{0}0.77 & \cc{16}0.65 & \cc{18}0.43 & \cc{15}24.48 & \cc{17}0.78 & \cc{21}0.05 & \cc{16}0.0709 & \cc{14}0.5532 & \cc{12}0.8148 & \cc{15}0.9731 & \cc{17}0.9712 & \cc{14}0.1924 & \cc{12}0.1075 \\
REPA-E & \cc{23}1.69 & \cc{8}219.3 & \cc{0}0.77 & \cc{25}\textbf{0.67} & \cc{18}0.45 & \cc{16}24.68 & \cc{17}0.78 & \cc{21}0.05 & \cc{16}0.0700 & \cc{14}0.5540 & \cc{12}0.8115 & \cc{15}0.9746 & \cc{16}0.9685 & \cc{17}0.1837 & \cc{5}0.1165 \\
RAE    & \cc{25}\textbf{1.51} & \cc{13}242.9 & \cc{25}\textbf{0.79} & \cc{8}0.63 & \cc{7}0.77 & \cc{0}18.41 & \cc{0}0.51 & \cc{0}0.15 & \cc{0}0.1092 & \cc{0}0.3096 & \cc{0}0.7437 & \cc{0}0.9440 & \cc{0}0.9326 & \cc{0}0.2364 & \cc{20}0.1038 \\
VTP-B  & \cc{0}3.88 & - & - & - & \cc{0}0.98 & \cc{13}23.85 & \cc{14}0.74 & \cc{15}0.08 & \cc{12}0.0836 & \cc{9}0.4750 & \cc{11}0.7999 & \cc{10}0.9616 & \cc{8}0.9487 & \cc{11}0.2043 & \cc{0}0.1213 \\
SVG    & \cc{5}3.36 & \cc{0}181.2 & - & - & \cc{3}0.89 & \cc{9}21.95 & \cc{9}0.65 & \cc{8}0.11 & \cc{4}0.0956 & \cc{4}0.3789 & \cc{11}0.7954 & \cc{5}0.9548 & \cc{0}0.9317 & \cc{8}0.2114 & \cc{25}\textbf{0.0996} \\
UAE    & \cc{23}1.68 & \cc{25}\textbf{301.6} & \cc{0}0.77 & \cc{0}0.61 & \cc{25}\textbf{0.24} & \cc{25}\textbf{28.26} & \cc{25}\textbf{0.91} & \cc{25}\textbf{0.03} & \cc{25}\textbf{0.0463} & \cc{25}\textbf{0.7434} & \cc{25}\textbf{0.8786} & \cc{25}\textbf{0.9902} & \cc{25}\textbf{0.9870} & \cc{25}\textbf{0.1543} & \cc{12}0.1057 \\
\bottomrule
\end{tabular}%
}
\vspace{-4mm}
\end{table*}

\begin{figure}[t!]
    \centering
    \includegraphics[width=\columnwidth]{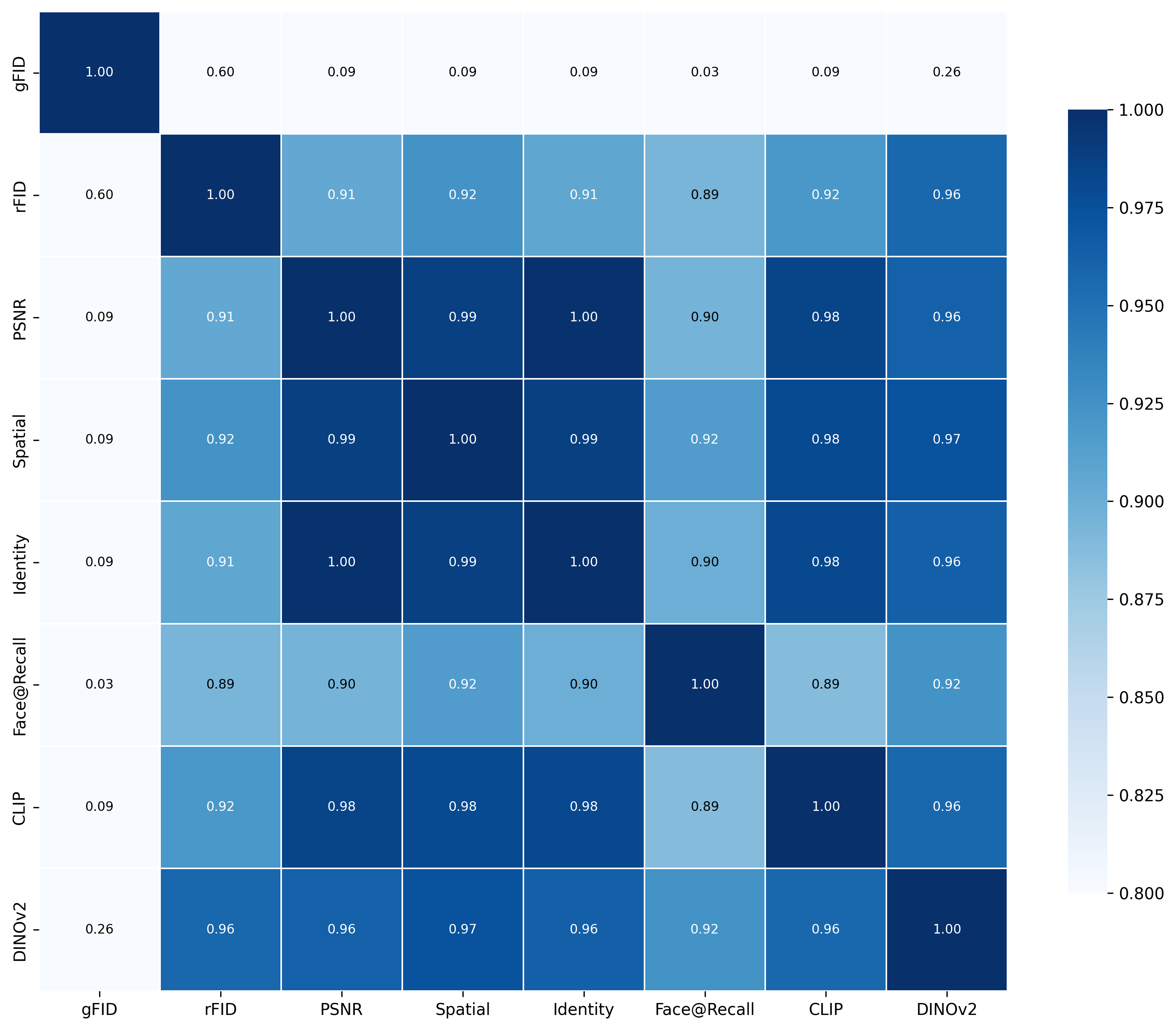}
    \caption{\textbf{Spearman correlation (absolute values) matrix across metrics.}
    gFID shows weak correlation with most instance-level and drift measures, while reconstruction metrics align more strongly. To make the contrast more apparent, we restrict the color scale to the range $0.8–1.0$.}
    \vspace{-5mm}
    \label{fig:spearman}
\end{figure}

\section{Empirical Study}
\label{sec:experiments}
Motivated by the fact that recent ImageNet autoencoder benchmarks often prioritize gFID while under-reporting instance-level reconstruction, and that our analysis suggests marginal metrics cannot certify coupling-dependent condition preservation, we conduct the empirical study below.

\subsection{Condition Drift on ImageNet Autoencoders}
\label{sec:gap_imagenet_eval}
% In this part, we connect these observations and insights by directly measuring condition drift across modern ImageNet autoencoders and assessing how well commonly reported metrics reflect it.

\paragraph{Experimental protocol.}
We evaluate a pool of modern ImageNet autoencoders using three categories of metrics:
(i) generation quality (\textit{Generation Metrics}),
(ii) reconstruction fidelity (\textit{Reconstruction Metrics}), and
(iii) condition drift across low-level structure and high-level semantics (\textit{Condition Consistency}). Specifically, \emph{Spatial} indicates the average drift among several spatial conditions (edge, depth, and segmentation). \emph{Identity} and \emph{Face@R} represent identity similarity and face detection recall, respectively. CLIP and DINO measure embedding similarity. Since CLIP aligns image embeddings with text, it also serves as a proxy for text-aligned semantic similarity between two images.

We evaluate six open-sourced AEs and their 33 variants released in the official codebases. For generation metrics, we use the values reported in the original papers and official reports. We compute reconstruction and condition-drift metrics for all open-sourced variants and report Spearman correlations (Figure~\ref{fig:spearman}) and scatter plots (Figure~\ref{fig:metric_vs_drift_scatter}). The complete model list and full results are provided in Table~\ref{tab:metrics}.
Full implementation details are provided in Appendix~\ref{app:exp_settings_benchmark}.

\paragraph{gFID is weakly aligned with condition drift.}
Table~\ref{tab:main_result} provides concrete evidence for the risk highlighted in Observation~2.
Although gFID is a meaningful indicator of unconditional generation quality, it does not reliably reflect whether an autoencoder preserves control-relevant structure.
For example, RAE attains the strongest gFID among the listed models, yet it exhibits substantially larger drift on several condition projectors.
In contrast, UAE maintains consistently strong condition preservation, even though its gFID is worse than RAE and comparable to REPA-E.
Together, these comparisons illustrate that gFID-centric selection can favor autoencoders that deviate more on target controls even when unconditional generation quality appears strong.

This conclusion is reinforced by a rank-based correlation analysis.
We compute Spearman correlation coefficients across the reported metrics, which capture the agreement between model rankings induced by different measurements, and visualize the relationships with scatter plots.
As shown in Figure~\ref{fig:spearman}, gFID exhibits consistently low correlation with most condition-drift measures, with coefficients around $0.1$ for the majority of projectors, indicating that gFID-based ranking provides little signal about condition preservation.
Figure~\ref{fig:metric_vs_drift_scatter} offers the same takeaway from a complementary view: models with stronger gFID can still incur larger drift, and the best-gFID region contains non-trivial violations of condition consistency.

These results imply that over-emphasizing generation metrics can be detrimental to controllable diffusion.
While such a trade-off may be easy to overlook on ImageNet when evaluation focuses on unconditional realism, it becomes increasingly consequential when scaling to large diffusion backbones and diverse downstream controls, where autoencoder drift can act as a bottleneck that limits generality.

\begin{figure*}[t!]
    \centering
    \begin{subfigure}[t]{\textwidth}
        \centering
        \includegraphics[width=\textwidth]{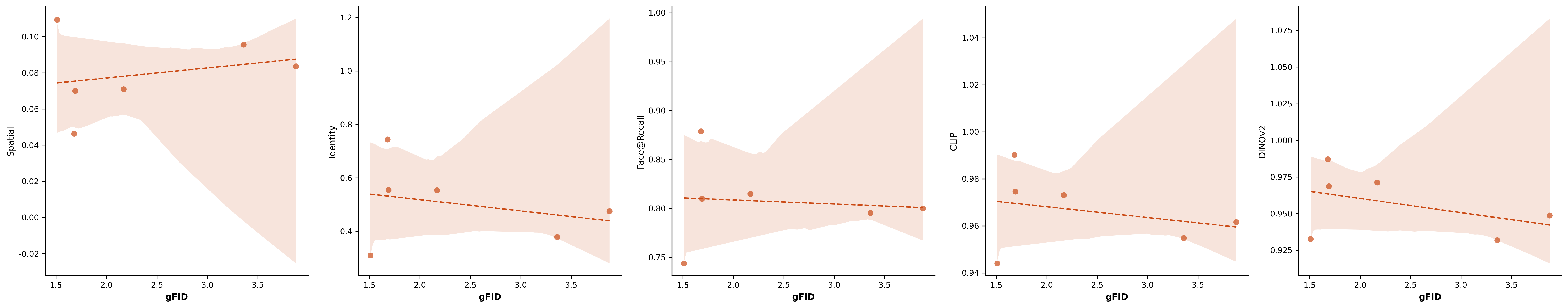}
        \caption{gFID (generation) vs.\ condition drift metrics.}
        \label{fig:gfid_scatter_sub}
    \end{subfigure}

    \begin{subfigure}[t]{\textwidth}
        \centering
        \includegraphics[width=\textwidth]{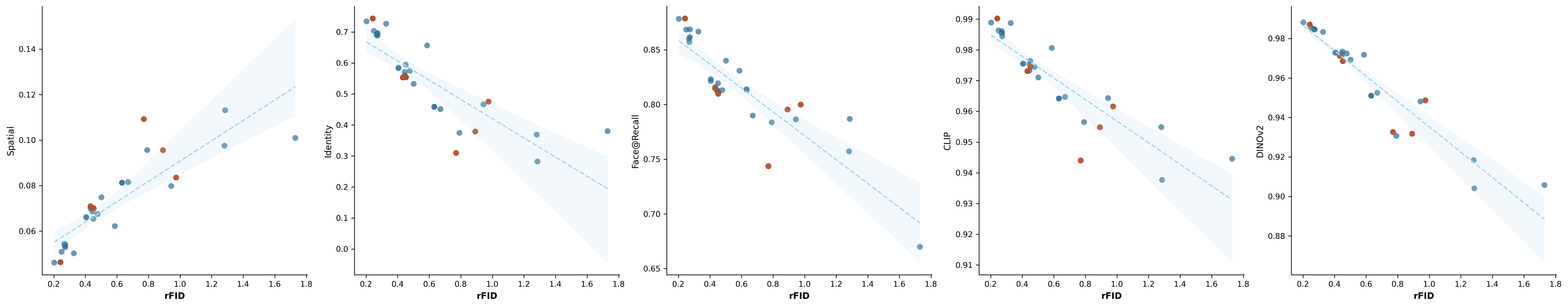}
        \caption{rFID (distribution-level reconstruction) vs.\ condition drift metrics.}
        \label{fig:rfid_scatter_sub}
    \end{subfigure}

    \vspace{0.6em}

    \begin{subfigure}[t]{\textwidth}
        \centering
        \includegraphics[width=\textwidth]{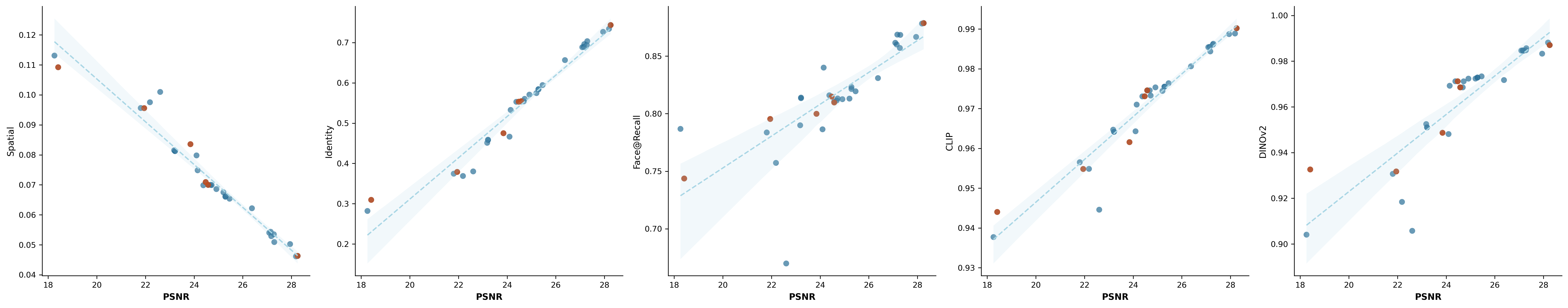}
        \caption{PSNR (instance-level reconstruction) vs.\ condition drift metrics.}
        \label{fig:psnr_scatter_sub}
    \end{subfigure}

    \caption{\textbf{Common metrics vs.\ condition drift.}
    gFID shows weak correlation with drift.
    rFID aligns with drift on average but remains incomplete as a distributional score.
    Instance-level fidelity correlates strongly with many drift measures and serves as a simple sanity signal.}
    \vspace{-4mm}
    \label{fig:metric_vs_drift_scatter}
\end{figure*}

\begin{finding}{gFID can mis-rank condition drift}{}
gFID weakly reflects condition drift and can mis-rank autoencoders for controllable generation.
\end{finding}

\paragraph{Reconstruction is more aligned with condition drift.}
As analyzed in \S~\ref{sec:cc_vs_metrics}, reconstruction metrics are theoretically better aligned with condition drift.
This trend is also visible in Table~\ref{tab:main_result}.
Autoencoders that reconstruct more faithfully at the instance level also tend to preserve control-relevant conditions more reliably.
UAE ranks best on the instance-level fidelity suite and it also yields the strongest condition consistency across both structural projectors and identity/semantic projectors, producing minimal drift.
By contrast, RAE attains the best gFID but its reconstructions are less faithful and its condition drift is noticeably larger on the same control projectors.
Despite weaker gFID than RAE, VA-VAE delivers materially better reconstruction fidelity and correspondingly smaller structural drift along with stronger face consistency. 
% Although

Beyond these headline models, we further expand the analysis to a more comprehensive setting, spanning 33 variants from these AE studies (total results are shown in Table~\ref{tab:metrics}).
Figure~\ref{fig:spearman} makes the contrast particularly clear.
Reconstruction metrics, including rFID and PSNR, exhibit substantially higher correlations with the drift measures we consider than gFID.
Moreover, PSNR is more rank-consistent with condition drift than rFID: across the drift measures we consider, PSNR achieves Spearman correlations above $0.96$, whereas rFID is typically around $0.9$.
The scatter plots in Figure~\ref{fig:metric_vs_drift_scatter} corroborate the same pattern.
Both reconstruction metrics provide a more faithful signal for condition drift than generation metrics, and PSNR further shows tighter alignment than rFID in terms of ranking consistency.
We also observe that although Canny edge is not a Lipschitz condition, it also can be measured by instance-level fidelity.

% \paragraph{Takeaway.} Across our evaluated autoencoders, reconstruction-oriented metrics are substantially more predictive of condition drift than gFID. In particular, rFID provides a strong directionally aligned signal, while instance-level fidelity (e.g., PSNR) is even more predictive and offers a stricter per-sample safeguard. This is consistent with our analysis: marginal metrics such as rFID can correlate with drift within a given model family, but they do not by themselves certify instance-wise structure preservation; therefore we recommend reporting rFID together with at least one instance-level metric, and measuring condition drift (or consistency) for target controls when available.

\begin{finding}{Reconstruction tracks condition drift}{}
Reconstruction metrics provide a more faithful assessment of condition drift than gFID, and instance-level fidelity is slightly more informative than rFID.
\end{finding}

\subsection{Further Explorations on Controllable Generation}
\label{sec:further_explorations}

The preceding analysis suggests that autoencoder condition drift can be substantial and is not reliably captured by generation-centric metrics.
In this subsection, we connect these diagnostic findings to controllable diffusion behavior through two complementary explorations that target distinct failure modes.
First, we train ControlNet on two representative ImageNet autoencoders with their corresponding diffusion backbones, and evaluate how the autoencoder affects controlled generation quality and condition adherence.
Second, we further probe the latent representation directly by training lightweight predictors to recover conditions from latents, testing whether control-relevant cues are already discarded at encoding time.
These experiments provide mechanistic evidence that autoencoder drift can materially hinder controllable-generation training and that part of this drift can originate from irreversible information loss during encoding.

\paragraph{Empirical study on ControlNet.}
To further validate how the choice of autoencoder affects controllable generation in practice, we train ControlNet using two representative autoencoders, VA-VAE and RAE, on both Canny-to-image and depth-to-image tasks.
Table~\ref{tab:controlnet_results} reports the resulting condition alignment metrics together with key autoencoder metrics for controlled sampling under each condition.
Although RAE achieves better unconditional ImageNet gFID than VA-VAE, its performance degrades markedly once ControlNet is trained for conditional generation.
Across both conditions, RAE exhibits worse controlled generation quality than VA-VAE, indicating that strong unconditional generation does not translate to reliable controllability when the autoencoder fails to preserve control-relevant structure.

Figure~\ref{fig:controlnet_qual} provides qualitative comparisons that mirror these quantitative results.
Generations based on VA-VAE are visually more coherent and follow the specified control signals more faithfully, whereas RAE more frequently violates the target structure and produces lower-quality outputs.
During training, we also observe that ControlNet optimization with RAE is substantially less stable and converges more slowly, consistent with the hypothesis that encoding-time information loss increases the difficulty of learning robust conditional mappings.

\begin{table}[t!]
\centering
\caption{\textbf{ControlNet studies.}
We train ControlNet on VA-VAE and RAE diffusion models and report condition alignment metrics together with key autoencoder metrics.
% Although RAE achieves better gFID, its controllable generation performance is markedly worse than VA-VAE under the same training setup.
}

\label{tab:controlnet_results}
\small
\renewcommand{\arraystretch}{1.15}
\resizebox{\linewidth}{!}{%
\begin{tabular}{c ccc cc cc}
\toprule
\multirow{2}{*}{\textbf{AE}} &
\multicolumn{3}{c}{\textbf{AE Metrics}} &
\multicolumn{2}{c}{\textbf{Canny-to-image}} &
\multicolumn{2}{c}{\textbf{Depth-to-image}} \\
\cmidrule(lr){2-4}\cmidrule(lr){5-6}\cmidrule(lr){7-8}
& gFID $\downarrow$ & rFID $\downarrow$ & PSNR $\uparrow$
& FID $\downarrow$ & L1 $\downarrow$
& FID $\downarrow$ & L1 $\downarrow$ \\
\midrule
VA-VAE & 2.17 & \textbf{0.43} & \textbf{24.48} & \textbf{11.59} & \textbf{0.1962} & \textbf{5.99} & \textbf{0.1280} \\
RAE    & \textbf{1.51} & 0.77 & 18.41 & 76.28 & 0.2021 & 51.02 & 0.1967 \\
\bottomrule
\end{tabular}%
}
\vspace{-5mm}
\end{table}

\begin{figure}[t!]
\centering
\includegraphics[width=\columnwidth]{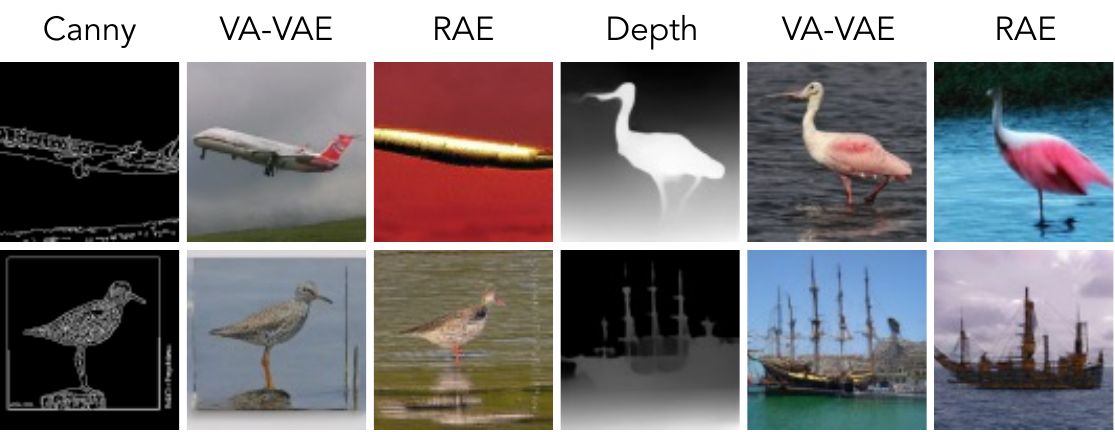}
\caption{\textbf{Qualitative controllable generation results.}
A grid comparing Canny-to-image and Depth-to-image outputs across different frozen autoencoders.}
\vspace{-6mm}
\label{fig:controlnet_qual}
\end{figure}

\begin{finding}{Control is constrained by AE drift}{}
Under identical ControlNet settings, controllability is constrained by autoencoder condition drift.
%\vspace{-3mm}
\end{finding}
\vspace{-4mm}

\paragraph{Latent-space condition prediction probes encoding-time information loss.}
We additionally probe whether condition information is already lost during encoding.
To this end, we train a lightweight predictor $h(\cdot)$ that infers conditions directly from latents $z=\mathcal{E}(x)$ for representative controls such as Canny edges and depth.
If the encoder discards condition cues, then even a reasonably expressive $h$ cannot recover them reliably.
Implementation details are provided in Appendix~\ref{app:exp_settings_probe}.
We run this probe on a subset of autoencoders and summarize the results in Table~\ref{tab:main_result} \emph{Latent Condition} part.
UAE again performs best, which is consistent with its strong condition preservation and suggests that its encoder retains control-relevant information in the latent representation.
RAE exhibits a contrasting behavior.
It performs competitively on depth prediction while being the weakest on Canny prediction, indicating an uneven preservation of condition cues at encoding time.
%One plausible explanation is that RAE relies on a frozen DINOv2 encoder, which tends to emphasize higher-level, geometry-related semantics while being less faithful to fine-grained low-level structures such as edges.

\begin{finding}{Encoding can lose cues}{}
For some autoencoders, conditions are less recoverable from latents right after encoding, suggesting early loss of control-relevant information.
\end{finding}

\section{Conclusion}

In this work, we revisit a tilted evaluation practice in recent autoencoder studies that over-emphasizes generation-centric metrics.
We show theoretically that de-emphasizing reconstruction weakens constraints on autoencoder-induced \emph{condition drift}, leading to poor preservation of condition information and degraded controllable generation.
Through an empirical study on modern ImageNet autoencoders and their open-sourced variants, we verify the prevalence of gFID-dominant selection and demonstrate its adverse impact on controllability.
We hope our findings broaden the perspective on autoencoder evaluation and help mitigate the gap between ImageNet-centric generation benchmarks and open-world controllable generation.

\newpage
\bibliography{example_paper}
\bibliographystyle{icml2026}

\newpage
\appendix
\onecolumn

\section{Experimental Settings}
\label{app:exp_settings}

\subsection{Condition Drift Evaluation (Section~\ref{sec:gap_imagenet_eval})}
\label{app:exp_settings_benchmark}

\subsubsection{Dataset and Pre-processing}
We conduct our evaluations on the ImageNet validation set (50,000 images). Unless otherwise specified, all reference and reconstructed images are center-cropped and resized to $256 \times 256$ resolution.

\subsubsection{Evaluation Protocol}
\label{app:bench_eval}

For generation evaluation, we borrow the results from their original reports.

\paragraph{Reconstruction evaluation.}
To assess the perceptual quality and diversity of the reconstructed images, we report the Fréchet Inception Distance (FID). We compute FID using a standard InceptionV3 network pretrained on ImageNet-1k \cite{szegedy2016rethinking}, calculating the distance between the feature statistics (mean and covariance) of the reference and reconstructed distributions. For instance-level reconstruction quality, we report PSNR and SSIM. Additionally, we utilize LPIPS (with the AlexNet backbone \cite{krizhevsky2012imagenet}) to measure perceptual similarity in the feature space.

\paragraph{Identity similarity and face detection recall.} For face-related evaluations, we utilize the InsightFace library (model \texttt{buffalo\_s}) \footnote{https://github.com/deepinsight/insightface}. We detect the largest face in both reference and reconstructed images and compute the cosine similarity between their L2-normalized identity embeddings. We also report the face detection ratio to ensure generation stability.

\paragraph{Embedding similarity.}
We measure the cosine similarity between image embeddings extracted by CLIP (ViT-B/32) \cite{radford2021learning} and DINOv2 (ViT-B/14). This assesses how well the reconstructed images capture the global semantic context of the reference. Moreover, CLIP embedding similarity can also be interpreted as a proxy for the similarity between the images' textual descriptions, since CLIP is trained with image--text contrastive learning to align images and their paired captions in a shared embedding space. Hence, two images that are close under cosine similarity tend to correspond to similar captions/prompts (i.e., they would retrieve similar text).

\paragraph{Spatial Control.}
To evaluate the fidelity of spatial conditions, we compute the $\ell_1$ distance between condition maps extracted from the reference and reconstructed images. We utilize a suite of off-the-shelf detectors to extract these maps, including: Canny for edge detection, MiDaS for depth estimation \cite{ranftl2020towards}, and Uniformer for semantic segmentation \cite{li2022uniformer}. All condition maps are normalized prior to distance calculation. The \emph{Spatial} metric is computed by averaging the condition drift over these three conditions.

\subsection{ControlNet Settings (Section~\ref{sec:further_explorations})}
\label{app:exp_settings_controlnet}

\begin{table}[h]
    \centering
    \caption{Training hyperparameters for ControlNet finetuning.}
    \label{tab:control_hyperparams}
    \begin{tabular}{lc}
        \toprule
        Hyperparameter & Value \\
        \midrule
        Base Batch Size & 4 \\
        Gradient Accumulation & 8 \\
        Effective Batch Size & 32 \\
        Learning Rate & $5 \times 10^{-5}$ \\
        Optimizer & AdamW \\
        Gradient Clipping & 0.5 \\
        Total Epochs & 100 \\
        \bottomrule
    \end{tabular}
\end{table}

\paragraph{Training Implementation}
We train the control models on the ImageNet training set using the AdamW optimizer. During this phase, the parameters of the main generative models (VAVAE or RAE) are frozen, and only the control branch parameters are updated. We employ a constant learning rate strategy without a scheduler. To stabilize training, we apply gradient clipping with a threshold of 0.5. We train until the performance converges on the ImageNet validation set. Training hyperparameters are summarized in Table~\ref{tab:control_hyperparams}. The training is conducted with mixed precision to optimize memory usage.

\subsection{Latent-space Probing (Section~\ref{sec:further_explorations})}
\label{app:exp_settings_probe}

\paragraph{Network Architecture}
To reconstruct spatial conditions from the latent representations, we implement a custom VGG-style decoder. The network accepts an input latent tensor of shape $(B, C, H, W)$ and first normalizes it via a 2D Batch Normalization layer. The backbone consists of four upsampling stages, where each stage sequentially applies nearest-neighbor upsampling with a scale factor of 2, followed by two blocks of $3 \times 3$ convolution, Batch Normalization, and ReLU activation. Through these stages, the channel dimension is progressively reduced following the sequence $C \to 128 \to 64 \to 32 \to 16$. The final output layer utilizes a $1 \times 1$ convolution to project the features into a single-channel map with a resolution of $(H \times 16, W \times 16)$. For depth estimation specifically, the output is passed through a Sigmoid activation function to constrain values to the $[0, 1]$ range.

\paragraph{Training Configuration}
We train the decoders for a maximum of 100 epochs using the AdamW optimizer with a learning rate of $10^{-4}$ and a batch size of 128. To prevent overfitting, we employ an early stopping strategy that terminates training if the validation loss does not decrease for 10 consecutive epochs. The loss functions are tailored to the specific modality. For edge detection tasks (e.g., Canny), we optimize a hybrid objective consisting of Binary Cross Entropy (BCE) and Dice loss, both weighted equally at 0.5. For depth estimation, we minimize a composite loss comprising an L1 pixel-wise loss and a gradient loss to ensure structural consistency, with the gradient term weighted by $\alpha = 0.1$.

\section{Proof of Theorem~\ref{thm:cc_floor} (Alignment Limit at Latent Optimum)}
\label{app:proofs}

\begingroup
\renewcommand{\thetheorem}{\ref{thm:cc_floor}} %

\begin{theorem}[Alignment Limit at Latent Optimum]
Fix an autoencoder $(\mathcal{E},\mathcal{D})$ and a condition projector $\phi$.
Assume the diffusion backbone perfectly fits the induced latent conditional distribution, i.e., $p_\theta(z\mid c) = p_{\mathcal{E}}(z\mid c)$, and generation follows $G_\theta(c)=\mathcal{D}(z)$ with $z\sim p_\theta(z\mid c)$.
Then the expected alignment error equals the expected autoencoder-induced condition drift:
\[
\mathcal{L}_{\mathrm{align}}(\theta)
=
\mathbb{E}_{x\sim p_{\text{data}}}\bigl[\Delta_{\mathrm{AE}}(x)\bigr].
\]
\end{theorem}
\addtocounter{theorem}{-1} %

\endgroup

\begin{proof}[Proof of Theorem~\ref{thm:cc_floor}]
We start from the definition of the alignment error:
\begin{equation}
\mathcal{L}_{\mathrm{align}}(\theta)
=
\mathbb{E}_{c\sim p(c)}\Bigl[\bigl\|\phi(G_\theta(c)) - c\bigr\|\Bigr],
\end{equation}
where generation is defined by drawing $z\sim p_\theta(z\mid c)$ and decoding $G_\theta(c)=\mathcal{D}(z)$.
Substituting this sampling procedure yields
\begin{equation}
\mathcal{L}_{\mathrm{align}}(\theta)
=
\mathbb{E}_{c\sim p(c)}\,\mathbb{E}_{z\sim p_\theta(z\mid c)}
\Bigl[\bigl\|\phi(\mathcal{D}(z)) - c\bigr\|\Bigr].
\label{eq:align_expand}
\end{equation}

By the latent-optimal assumption in Theorem~\ref{thm:cc_floor}, the diffusion backbone perfectly fits the induced latent conditional distribution,
\begin{equation}
p_\theta(z\mid c)=p_{\mathcal{E}}(z\mid c).
\label{eq:latent_opt_assump}
\end{equation}
Applying~\eqref{eq:latent_opt_assump} to~\eqref{eq:align_expand} gives
\begin{equation}
\mathcal{L}_{\mathrm{align}}(\theta)
=
\mathbb{E}_{c\sim p(c)}\,\mathbb{E}_{z\sim p_{\mathcal{E}}(z\mid c)}
\Bigl[\bigl\|\phi(\mathcal{D}(z)) - c\bigr\|\Bigr].
\label{eq:align_induced}
\end{equation}

Next, we use the definition of the induced latent conditional distribution $p_{\mathcal{E}}(z\mid c)$.
By construction, image-grounded pairs are generated by sampling $x\sim p_{\text{data}}(x)$ and setting
\begin{equation}
c=\phi(x),\qquad z=\mathcal{E}(x).
\label{eq:pair_def}
\end{equation}
Let $p(c)$ be the marginal distribution of $c=\phi(x)$ under $x\sim p_{\text{data}}(x)$.
Then the joint distribution of $(c,z)$ induced by~\eqref{eq:pair_def} admits the standard factorization
\begin{equation}
p(c,z)=p(c)\,p_{\mathcal{E}}(z\mid c).
\label{eq:cz_factor}
\end{equation}
Therefore, sampling $c\sim p(c)$ and then $z\sim p_{\mathcal{E}}(z\mid c)$ is equivalent to sampling $(c,z)\sim p(c,z)$, which in turn is equivalent to sampling $x\sim p_{\text{data}}(x)$ and setting $(c,z)$ via~\eqref{eq:pair_def}.
Using this equivalence to rewrite~\eqref{eq:align_induced}, we obtain
\begin{align}
\mathcal{L}_{\mathrm{align}}(\theta)
&=
\mathbb{E}_{x\sim p_{\text{data}}(x)}
\Bigl[\bigl\|\phi(\mathcal{D}(\mathcal{E}(x))) - \phi(x)\bigr\|\Bigr].
\label{eq:align_x}
\end{align}
Finally, by the definition of autoencoder-induced condition drift,
\begin{equation}
\Delta_{\mathrm{AE}}(x)=\|\phi(x)-\phi(\mathcal{D}(\mathcal{E}(x)))\|,
\end{equation}
so~\eqref{eq:align_x} is exactly
\begin{equation}
\mathcal{L}_{\mathrm{align}}(\theta)
=
\mathbb{E}_{x\sim p_{\text{data}}(x)}\bigl[\Delta_{\mathrm{AE}}(x)\bigr],
\end{equation}
which proves~\eqref{eq:cc_floor_eq}.
\end{proof}

\section{Additional Results}
\subsection{Evaluation results on all variants}

We present all results in Table~\ref{tab:metrics}. We omit variants with substantially worse performance (rFID $> 3.0$), as their reconstruction quality is severely degraded. gFID results are taken from official reports, and all other metrics are computed following Appendix~\ref{app:bench_eval}.
We also provide scatter plots in Figures~\ref{fig:total_gfid}--\ref{fig:total_lpips} to illustrate how generation and reconstruction metrics relate to condition drift.
Overall, these results suggest that reconstruction-focused metrics provide a more reliable safeguard for controllability than gFID-dominant model selection.

\begin{table}[t]
\centering
\caption{\textbf{Total results among all variants.} We omit some variants with substantially weaker performance (i.e., rFID $> 3.00$).}
\label{tab:metrics}
\resizebox{\textwidth}{!}{%
\begin{tabular}{ccccccccccccccc}
\toprule
Group & Weight Name & gFID & rFID & PSNR & SSIM & LPIPS & Canny & Depth & Seg & Spatial & Identity & Face@R & CLIP & DINOv2 \\
\midrule
\multirow{8}{*}{VA-VAE} & ldm\_f16d64\_50ep & - & 0.2731 & 27.0905 & 0.8739 & 0.0306 & 0.0993 & 0.0212 & 0.0415 & 0.0540 & 0.6884 & 0.8616 & 0.9854 & 0.9846 \\
 & vavae\_f16d64\_mae\_50ep & - & 0.2696 & 27.2789 & 0.8730 & 0.0306 & 0.0986 & 0.0209 & 0.0410 & 0.0535 & 0.6939 & 0.8571 & 0.9861 & 0.9848 \\
 & vavae\_f16d64\_dinov2\_50ep & - & 0.2674 & 27.1483 & 0.8716 & 0.0304 & 0.1007 & 0.0211 & 0.0413 & 0.0543 & 0.6891 & 0.8601 & 0.9855 & 0.9848 \\
 & ldm\_f16d32\_50ep & - & 0.4777 & 25.2062 & 0.8007 & 0.0478 & 0.1213 & 0.0264 & 0.0549 & 0.0675 & 0.5745 & 0.8131 & 0.9745 & 0.9724 \\
 & vavae\_f16d32\_mae\_50ep & - & 0.4449 & 24.9164 & 0.7931 & 0.0488 & 0.1235 & 0.0268 & 0.0555 & 0.0686 & 0.5707 & 0.8126 & 0.9753 & 0.9724 \\
 & vavae\_f16d32\_dinov2\_50ep & - & 0.4461 & 24.7204 & 0.7834 & 0.0515 & 0.1272 & 0.0273 & 0.0552 & 0.0699 & 0.5603 & 0.8132 & 0.9733 & 0.9712 \\
 & vavae\_f16d32 & 2.17 & 0.4329 & 24.4762 & 0.7761 & 0.0540 & 0.1285 & 0.0275 & 0.0568 & 0.0709 & 0.5532 & 0.8148 & 0.9731 & 0.9712 \\
 & ldm\_f16d16\_50ep & - & 0.6711 & 23.1802 & 0.7187 & 0.0697 & 0.1418 & 0.0334 & 0.0694 & 0.0815 & 0.4512 & 0.7899 & 0.9647 & 0.9525 \\
\midrule
\multirow{1}{*}{RAE} & rae\_dinov2\_reconstruction & 1.51 & 0.7704 & 18.4111 & 0.5074 & 0.1538 & 0.1835 & 0.0520 & 0.0921 & 0.1092 & 0.3096 & 0.7437 & 0.9440 & 0.9326 \\
\midrule
\multirow{3}{*}{VTP} & vtp\_s & - & 1.2814 & 22.1843 & 0.6621 & 0.1108 & 0.1557 & 0.0453 & 0.0917 & 0.0975 & 0.3688 & 0.7573 & 0.9548 & 0.9184 \\
 & vtp\_b & 3.88 & 0.9754 & 23.8501 & 0.7367 & 0.0796 & 0.1396 & 0.0349 & 0.0762 & 0.0836 & 0.4750 & 0.7999 & 0.9616 & 0.9487 \\
 & vtp\_l & - & 0.5019 & 24.1454 & 0.7579 & 0.0607 & 0.1369 & 0.0297 & 0.0580 & 0.0748 & 0.5328 & 0.8400 & 0.9710 & 0.9692 \\
\midrule
\multirow{4}{*}{UAE} & uae-stage2 & - & 0.5869 & 26.3785 & 0.8491 & 0.0569 & 0.1046 & 0.0283 & 0.0537 & 0.0622 & 0.6564 & 0.8309 & 0.9806 & 0.9717 \\
 & uae-stage3 & - & 0.3274 & 27.9472 & 0.8919 & 0.0368 & 0.0881 & 0.0209 & 0.0417 & 0.0502 & 0.7265 & 0.8667 & 0.9887 & 0.9833 \\
 & uae-stage4 & - & 0.2424 & 28.2569 & 0.9056 & 0.0257 & 0.0845 & 0.0188 & 0.0355 & 0.0463 & 0.7434 & 0.8786 & 0.9902 & 0.9870 \\
 & uae & 1.68 & 0.2424 & 28.2569 & 0.9056 & 0.0257 & 0.0845 & 0.0188 & 0.0355 & 0.0463 & 0.7434 & 0.8786 & 0.9902 & 0.9870 \\
\midrule
\multirow{13}{*}{REPA-E} & e2e-flux-vae & - & 0.2031 & 28.1893 & 0.9028 & 0.0248 & 0.0846 & 0.0185 & 0.0353 & 0.0461 & 0.7340 & 0.8783 & 0.9889 & 0.9881 \\
 & e2e-invae & - & 0.4056 & 25.2877 & 0.8053 & 0.0499 & 0.1183 & 0.0262 & 0.0536 & 0.0660 & 0.5843 & 0.8232 & 0.9755 & 0.9728 \\
 & e2e-invae-hf & - & 0.4057 & 25.2877 & 0.8053 & 0.0499 & 0.1183 & 0.0263 & 0.0539 & 0.0662 & 0.5836 & 0.8214 & 0.9755 & 0.9728 \\
 & e2e-qwenimage-vae & - & 0.2500 & 27.2968 & 0.8755 & 0.0316 & 0.0928 & 0.0199 & 0.0400 & 0.0509 & 0.7034 & 0.8684 & 0.9863 & 0.9857 \\
 & e2e-sd3.5-vae & - & 0.2734 & 27.1736 & 0.8748 & 0.0330 & 0.0947 & 0.0219 & 0.0421 & 0.0529 & 0.6961 & 0.8687 & 0.9844 & 0.9845 \\
 & e2e-sdvae & - & 0.6322 & 23.2145 & 0.7217 & 0.0730 & 0.1409 & 0.0330 & 0.0699 & 0.0813 & 0.4581 & 0.8135 & 0.9642 & 0.9510 \\
 & e2e-vavae & - & 0.4528 & 24.5796 & 0.7731 & 0.0551 & 0.1270 & 0.0274 & 0.0556 & 0.0700 & 0.5534 & 0.8106 & 0.9746 & 0.9685 \\
 & e2e-sdvae-hf & - & 0.6326 & 23.2145 & 0.7217 & 0.0730 & 0.1409 & 0.0332 & 0.0697 & 0.0812 & 0.4588 & 0.8140 & 0.9642 & 0.9510 \\
 & e2d-vavae-hf & 1.69 & 0.4527 & 24.5797 & 0.7731 & 0.0551 & 0.1270 & 0.0272 & 0.0558 & 0.0700 & 0.5544 & 0.8097 & 0.9746 & 0.9685 \\
 & invae & - & 0.4514 & 25.4580 & 0.8126 & 0.0486 & 0.1141 & 0.0258 & 0.0562 & 0.0654 & 0.5944 & 0.8194 & 0.9764 & 0.9734 \\
 & sdvae & - & 0.9442 & 24.0991 & 0.7345 & 0.0768 & 0.1295 & 0.0346 & 0.0754 & 0.0798 & 0.4664 & 0.7864 & 0.9643 & 0.9481 \\
 & vavae & - & 0.4337 & 24.3791 & 0.7736 & 0.0557 & 0.1261 & 0.0278 & 0.0557 & 0.0699 & 0.5529 & 0.8159 & 0.9731 & 0.9712 \\
 & e2e-vae & - & 0.4528 & 24.6827 & 0.7773 & 0.0532 & 0.1270 & 0.0274 & 0.0556 & 0.0700 & 0.5540 & 0.8115 & 0.9746 & 0.9685 \\
\midrule
\multirow{1}{*}{SVG} & svg & 3.36 & 0.8920 & 21.9476 & 0.6457 & 0.1146 & 0.1585 & 0.0445 & 0.0837 & 0.0956 & 0.3789 & 0.7954 & 0.9548 & 0.9317 \\
\midrule
\multirow{3}{*}{SVG-T2I} & svg\_t2i\_P\_stage1\_256 & - & 1.2857 & 18.2579 & 0.5059 & 0.1736 & 0.1803 & 0.0582 & 0.1008 & 0.1131 & 0.2821 & 0.7868 & 0.9377 & 0.9041 \\
 & svg\_t2i\_R\_stage1\_256 & - & 0.7918 & 21.8055 & 0.6523 & 0.1147 & 0.1619 & 0.0436 & 0.0813 & 0.0956 & 0.3745 & 0.7837 & 0.9565 & 0.9307 \\
 & svg\_t2i\_R\_stage2\_512 & - & 1.7300 & 22.6055 & 0.6801 & 0.1482 & 0.1511 & 0.0536 & 0.0983 & 0.1010 & 0.3802 & 0.6699 & 0.9446 & 0.9057 \\
\bottomrule
\end{tabular}%
}
\end{table}

\begin{figure}[t!]
    \centering
    \includegraphics[width=\columnwidth]{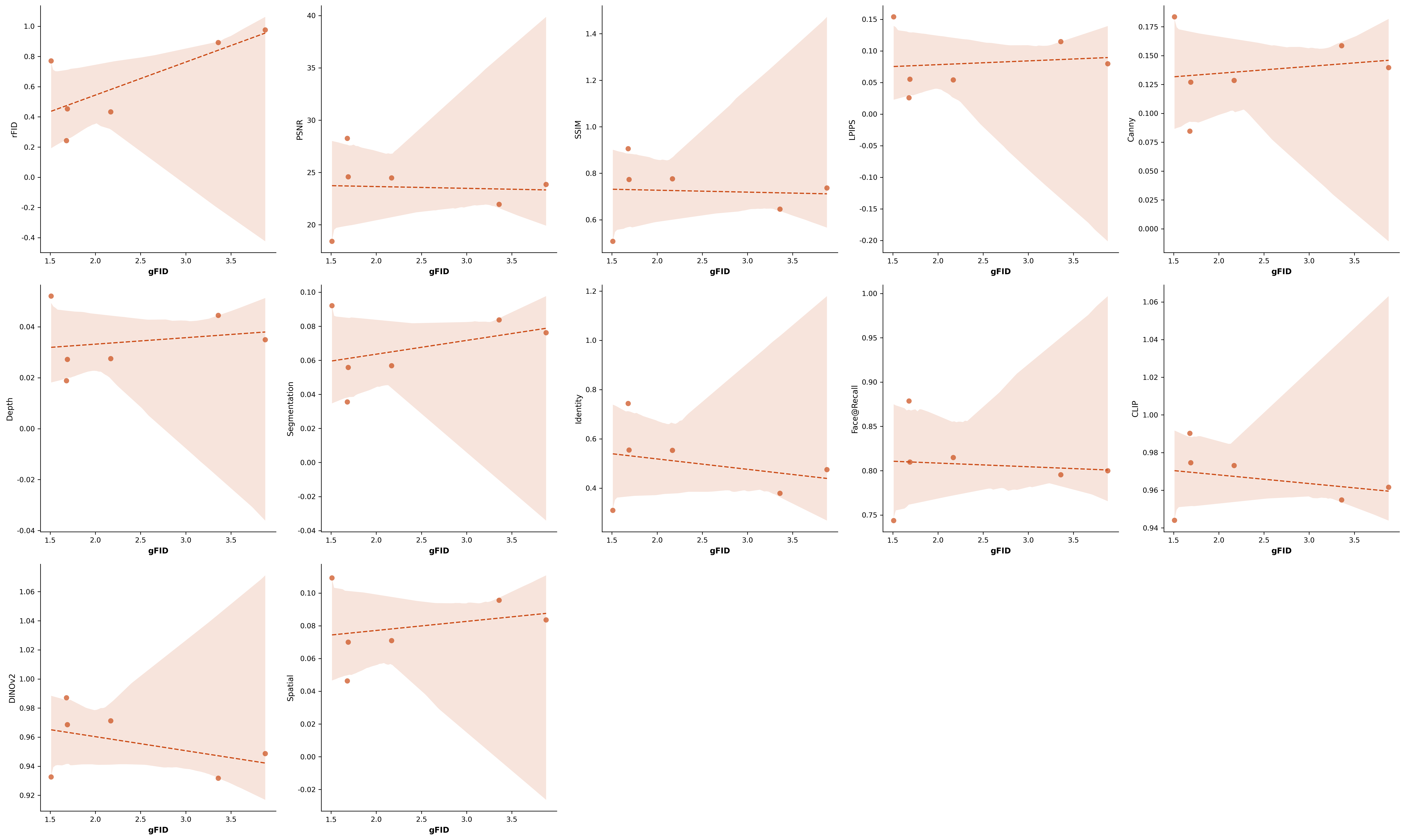}
    \caption{\textbf{Scatter plots between all metrics and gFID.}}
    \label{fig:total_gfid}
\end{figure}

\begin{figure}[t!]
    \centering
    \includegraphics[width=\columnwidth]{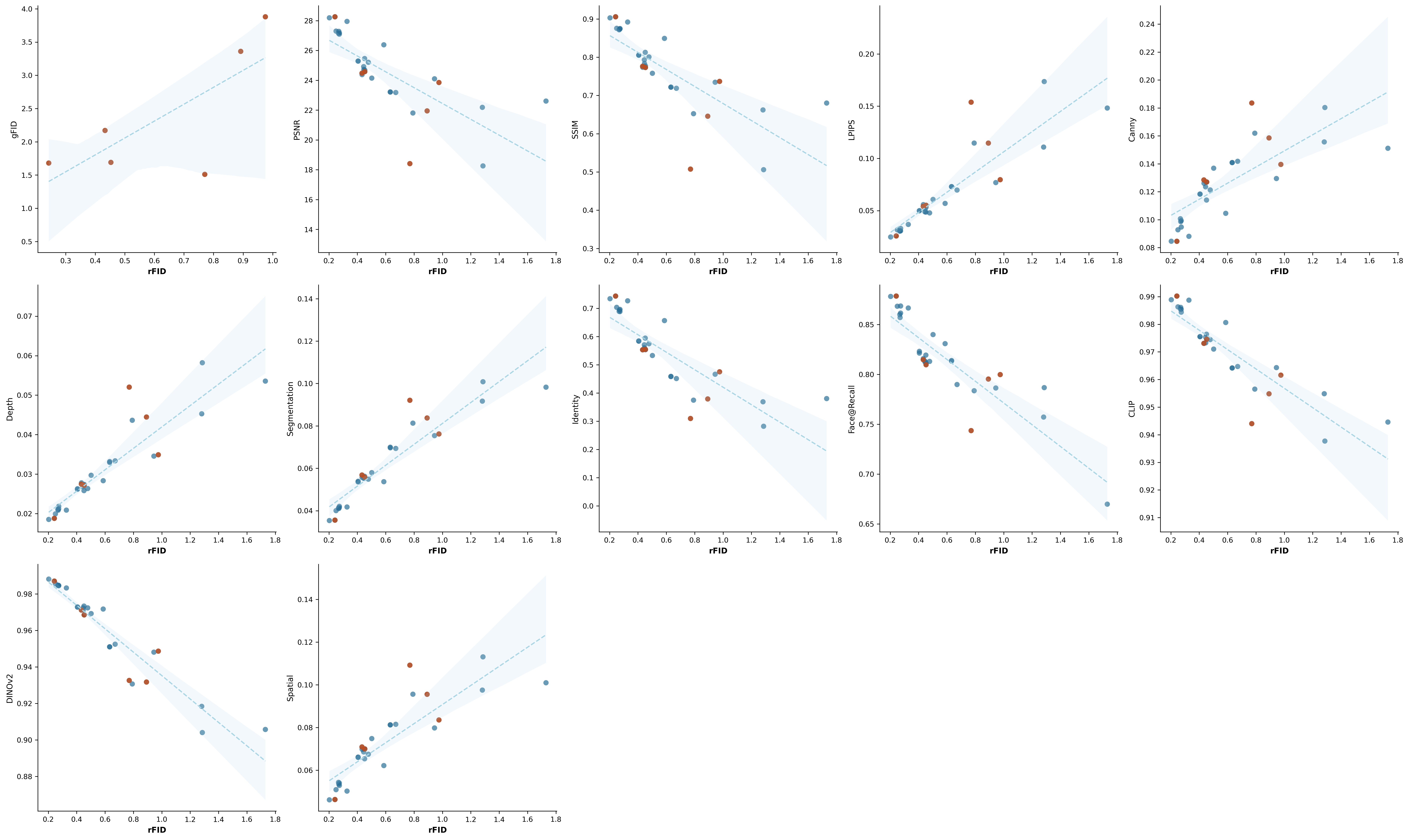}
    \caption{\textbf{Scatter plots between all metrics and rFID.}}
    \label{fig:total_rfid}
\end{figure}

\begin{figure}[t!]
    \centering
    \includegraphics[width=\columnwidth]{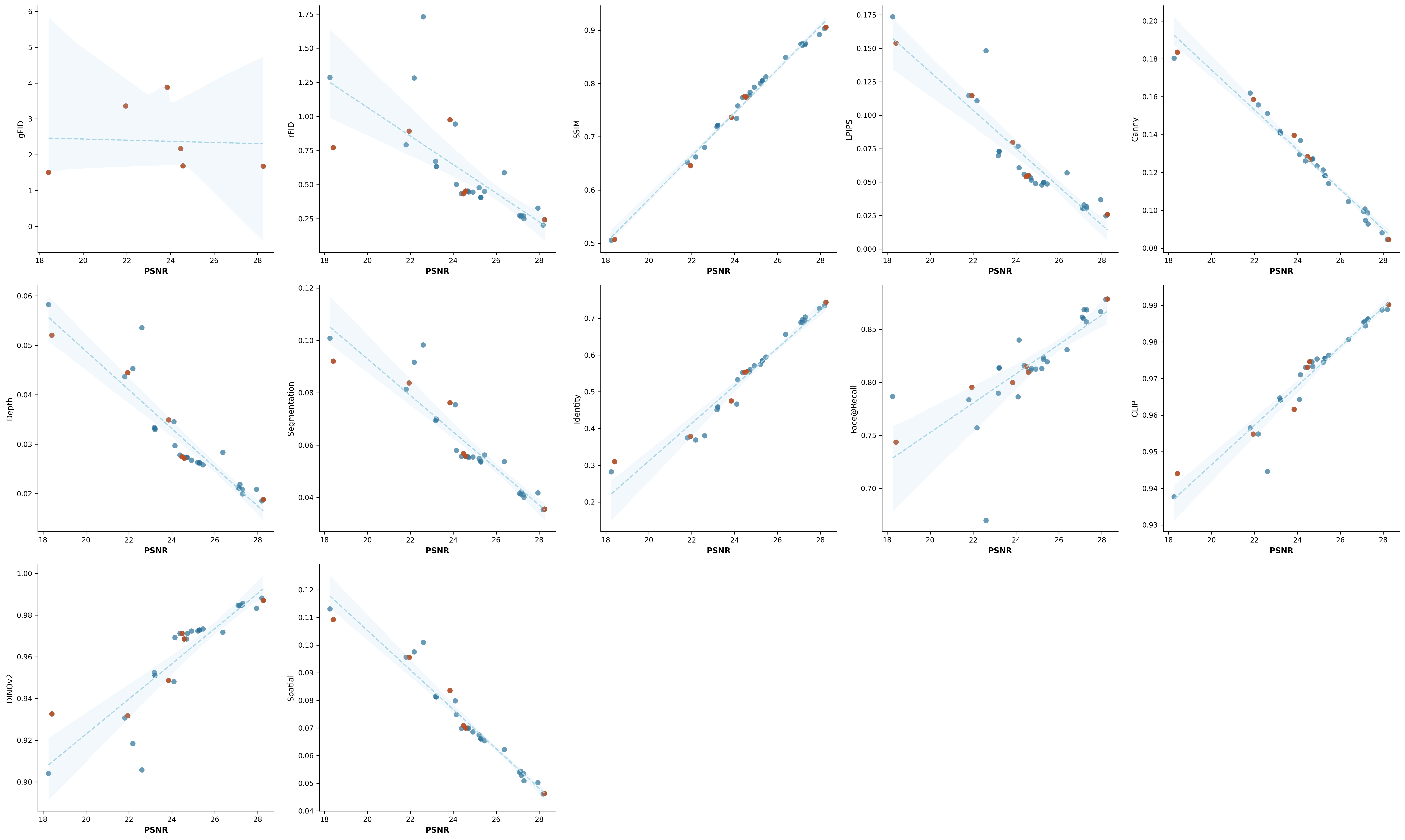}
    \caption{\textbf{Scatter plots between all metrics and PSNR.}}
    \label{fig:total_psnr}
\end{figure}

\begin{figure}[t!]
    \centering
    \includegraphics[width=\columnwidth]{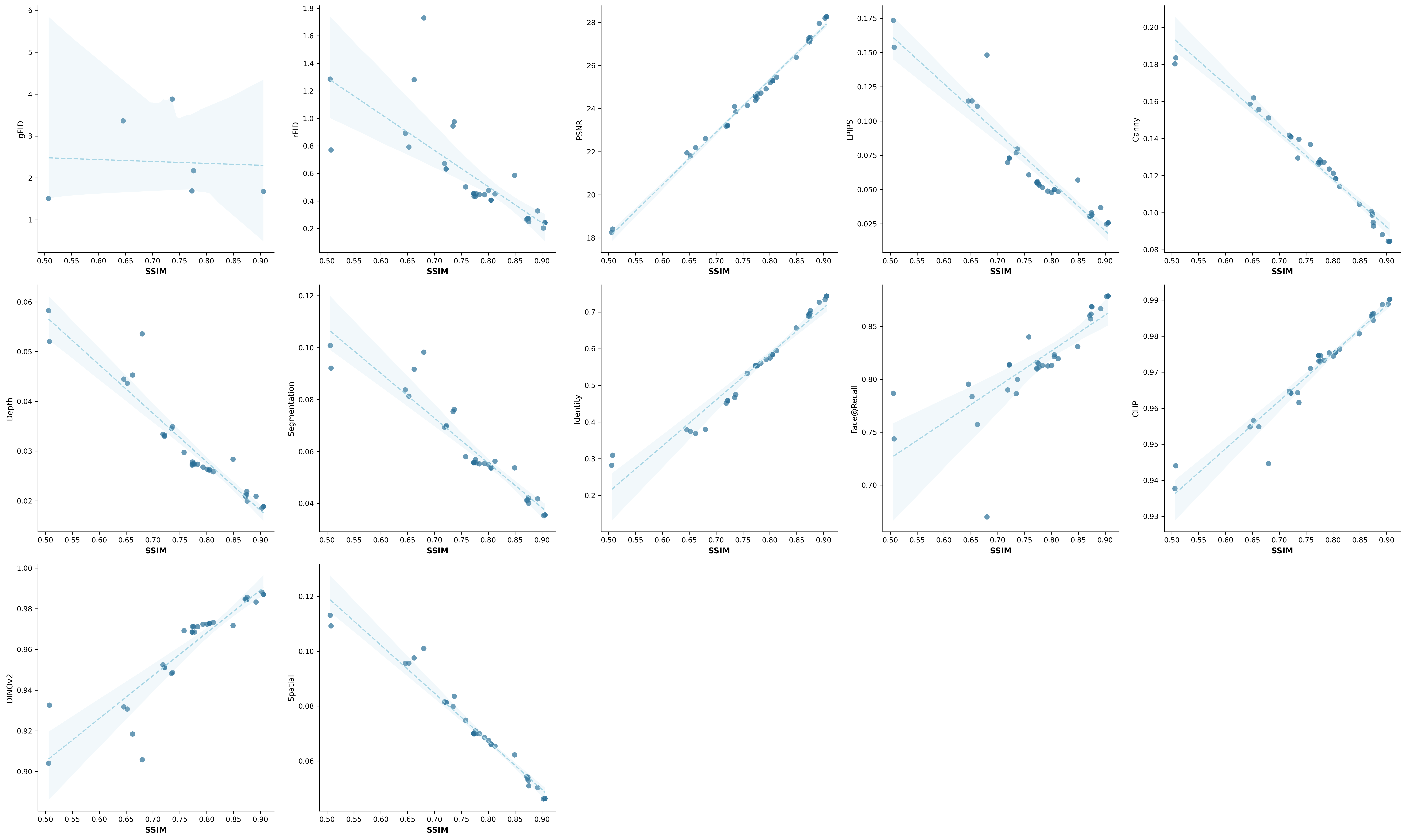}
    \caption{\textbf{Scatter plots between all metrics and SSIM.}}
    \label{fig:total_ssim}
\end{figure}

\begin{figure}[t!]
    \centering
    \includegraphics[width=\columnwidth]{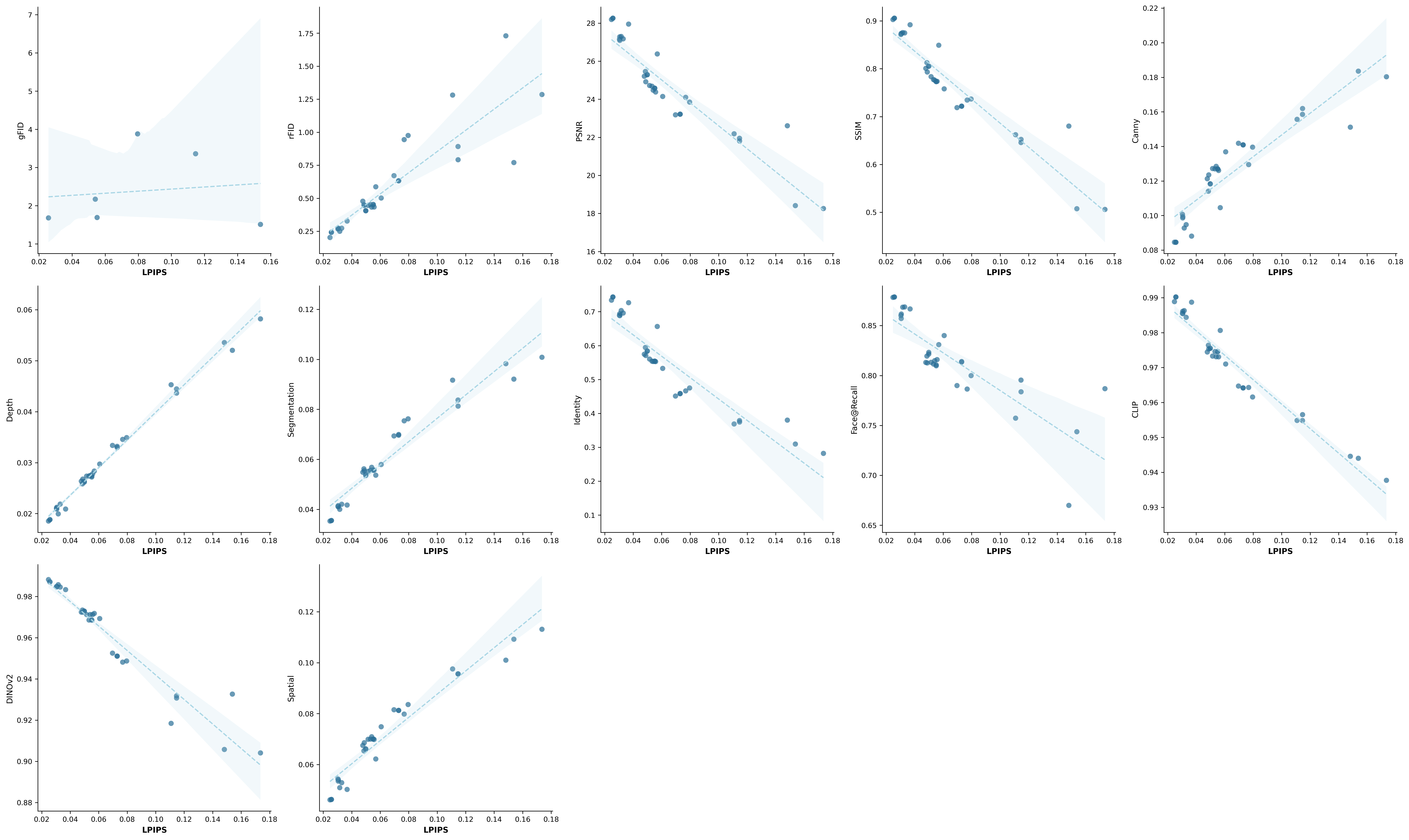}
    \caption{\textbf{Scatter plots between all metrics and LPIPS.}}
    \label{fig:total_lpips}
\end{figure}

\end{document}